\documentclass[11pt]{article}

\usepackage{amssymb}
\usepackage{amsmath}
\usepackage{amsthm}
\usepackage{xspace}
\usepackage{anysize}
\usepackage{comment}
\usepackage{url}
\usepackage{hyperref}
\usepackage{natbib}
\usepackage{color}

\DeclareMathOperator*{\argmin}{arg\,min}

\DeclareMathOperator*{\E}{\mathbb{E}}

\DeclareMathOperator{\sign}{sign}

\newcommand{\eps}{\epsilon}

\newtheorem {theorem}{Theorem}

\newtheorem {corollary}[theorem]{Corollary}

\newcommand{\eqr}[1]{Eq.~\eqref{eq:#1}}
\newcommand {\norm}[1]{\ensuremath{\| #1 \|}}
\newcommand {\paren}[1]{\ensuremath{\left(#1\right)}}
\newcommand {\set}[1]{\ensuremath{\left\{#1\right\}}}
\newcommand{\BO}{\mathcal{O}}
\newcommand{\R}{\ensuremath{\mathbb{R}}}

\newcommand{\abs}[1]{|#1|}

\newcommand{\Regret}{\operatorname{Regret}}
\newcommand{\Reward}{\text{Reward}}
\newcommand{\Loss}{\text{Loss}}
\newcommand{\Winnings}{\text{Winnings}}

\renewcommand{\Pr}{\operatorname{Pr}}
\newcommand{\spm}{\set{-1,1}}
\newcommand{\ti}{_{t+1}}

\newcommand{\xs}{\mathring{x}}
\newcommand {\bexp}[1]{\exp\left(#1\right)}
\newcommand {\bef}[2]{\exp\left(\frac{#1}{#2}\right)}

\newcommand{\rep}[1]{\left\langle #1 \right\rangle}
\newcommand {\p}[1]{\left(#1\right)}
\newcommand{\X}{\mathcal{X}}
\newcommand{\G}{\mathcal{G}}

\newcommand{\PB}{\mathcal{B}}

\newcommand{\h}{\frac{1}{2}}
\newcommand{\pfrac}[2]{\left(\frac{#1}{#2}\right)}

\definecolor{darkgreen}{rgb}{0,0.4,0.0}

\usepackage{anysize}
\marginsize{3cm}{2.5cm}{3cm}{2.5cm}

\newcommand{\arxivonly}[1]{#1}

\title{Minimax Optimal Algorithms \\
for Unconstrained Linear Optimization}

\author{H. Brendan McMahan \\
Google, Inc. \\
\texttt{\small{mcmahan@google.com}}
}

\begin{document}
\maketitle

\begin{abstract}
  We design and analyze minimax-optimal algorithms for online linear
  optimization games where the player's choice is unconstrained.  The
  player strives to minimize regret, the difference between his loss
  and the loss of a post-hoc benchmark strategy.  The standard
  benchmark is the loss of the best strategy chosen from a bounded
  comparator set.  When the the comparison set and the adversary's
  gradients satisfy $L_\infty$ bounds, we give the value of the game
  in closed form and prove it approaches $\sqrt{2 T / \pi}$ as $T
  \rightarrow \infty$.

  Interesting algorithms result when we consider soft constraints on
  the comparator, rather than restricting it to a bounded set.  As a
  warmup, we analyze the game with a quadratic penalty.  The value of
  this game is exactly $T/2$, and this value is achieved by perhaps
  the simplest online algorithm of all: unprojected gradient descent
  with a constant learning rate.
  We then derive a minimax-optimal algorithm for a much softer penalty
  function.  This algorithm achieves good bounds under the standard
  notion of regret for any comparator point, without needing to
  specify the comparator set in advance.  The value of this game
  converges to $\sqrt{e}$ as $T \rightarrow \infty$; we give a
  closed-form for the exact value as a function of $T$.  The resulting
  algorithm is natural in unconstrained investment or betting
  scenarios, since it guarantees at worst constant loss, while
  allowing for exponential reward against an ``easy'' adversary.
\end{abstract}

\section{Introduction}
Minimax analysis has recently been shown to be a powerful tool for the
construction of online learning algorithms~\citep{rakhlin12relax}.
Generally, these results use bounds on the value of the game (often
based on the sequential Rademacher complexity) in order to construct
efficient algorithms.  In this work, we show that when the learner is
unconstrained, it is often possible to efficiently compute an exact
minimax strategy.

We consider a game where on each round
$t=1, \dots, T$, first the learner selects $x_t \in \R^n$, and then an adversary
chooses $g_t \in \G \subset \R^n$, and the learner suffers loss $g_t
\cdot x_t$.
The goal of the learner is to minimize regret, that is, loss in excess
of that achieved by a benchmark strategy.  We define
\begin{equation}\label{eq:linreg}
 \Regret = \Loss - (\text{Benchmark Loss})
         = \sum_{t=1}^T g_t \cdot x_t - L(g_1, \dots, g_T)
\end{equation}
as the regret with respect to benchmark performance $L$ (the $L$
intended will be clear from context).  Letting $I(x \in \X) = 0$ for
$x \in \X$ and $\infty$ otherwise, the standard definition of regret
arises from the choice
\begin{equation}\label{eq:stdL}
 L(g_1, \dots, g_T) = \inf_{x \in \R^n} g_{1:T} \cdot x + I(x \in \X),
\end{equation}
the loss of the best strategy in a bounded convex set $\X$ (we write
$g_{1:t} = \sum_{s=1}^t g_s$ for a sum of scalars or vectors).  When
$L$ depends only on the sum $G \equiv g_{1:T}$ we write $L(G)$.  We
will be able to interpret the alternative benchmarks $L$ we consider
as penalties $\Psi$ on comparator points, so $L(G) = \argmin_x G \cdot
x + \Psi(x)$, where $\Psi(x)$ has replaced $I(x \in \X)$ in
\eqr{stdL}.

We view this interaction as a sequential zero-sum game played over $T$
rounds, where the player strives to minimize \eqr{linreg}, and the
adversary attempts to maximize it.  We study the value of this game,
$V^T$, and design minimax optimal algorithms for the player; formal
definitions are given below.  Some results are more naturally stated
in terms of rewards rather than losses, and so we define $\Reward = -
\Loss = -\sum_{t=1}^T g_t x_t$.

\paragraph{Outline and Summary of Results} Section~\ref{sec:regret}
provides motivation for the consideration of alternative benchmarks
$L$.  Section~\ref{sec:gen} then develops several theoretical tools
for analyzing unconstrained games with concave benchmark functions
$L$.  Section~\ref{sec:algs} applies this theory to three
particular instances; Figure~\ref{fig:results} summarizes the results
from this section.  These games exhibit a strong combinatorial
structure, which leads to interesting algorithms and perhaps
surprising game values.

Section~\ref{sec:gd} serves as a warmup, where we
show that constant step-size gradient descent is in fact minimax
optimal for a natural choice of $L$, which can be though of as
replacing the hard feasible set $\X$ in \eqr{stdL} with a quadratic
penalty function on comparator points.
Section~\ref{sec:hypercube} provides results analogous to those
of~\citet{abernethy08}; we consider regret compared to the best $\xs$
where $\norm{\xs}_\infty \le 1$ against an adversary constrained to
play $\norm{g_t}_\infty \le 1$, while Abernethy et al. considered
$\norm{g_t}_2 \le 1$ and $\norm{\xs}_2 \le 1$ for $n \ge 3$
dimensions.  Interestingly, while we prove results for the
unconstrained player, we show the optimal strategy in fact always
plays points from $\X = \{x \mid \norm{x}_\infty \le 1\}$, and so
applies to the constrained case as well.  Our results hold for the
$n=1$ case (where $L_2$ and $L_\infty$ coincide), showing that the
value of the game approaches $\sqrt{2T / \pi}$ as $T \rightarrow
\infty$, as opposed to $\sqrt{T}$ as one might extrapolate from
the results of Abernethy.  This indicates an interesting change in the
geometry of the $L_2$ game between $n=1$ and $n=3$.
Finally, Section~\ref{sec:inclr} gives a minimax optimal algorithm for
the setting introduced by~\cite{streeter12unconstrained}.  Following
their work, our algorithm obtains standard regret at most $\BO(R \sqrt
T \log\p{(1+R)T})$ simultaneously for any comparator $\xs$ with
$\abs{\xs} = R$, without needing to choose $R$ in advance.  However,
we emphasize a slightly different interpretation of this setting,
discussed in Section~\ref{sec:regret}.  It is worth noting that the
regret (relative the the respective $L$) of these algorithms is
$\BO(T)$, $\BO(\sqrt{T})$, and $\BO(1)$, respectively, though all
three are minimax algorithms.

\newcommand{\T}{\rule{0pt}{3.4ex}}
\begin{figure}
\begin{center}
\arxivonly{\begin{small}}
\renewcommand{\arraystretch}{1.5}
\begin{tabular}{|lllll|}
\hline
setting            & $L(G)$  & $\Psi(\xs)$
    & minimax value & update \\

\hline
soft feasible set  & $-\frac{G^2}{2 \sigma}$ & $\frac{\sigma}{2}\xs^2$
  & $\frac{T}{2\sigma}$ & $x\ti = \frac{1}{\sigma}g_{1:t}$ \\
standard regret    & \T $-\abs{G}$  & $I(\abs{\xs} \leq 1)$
   &  $\rightarrow \sqrt{\frac{2}{\pi} T}$  & \eqr{salg}\\
bounded-loss betting%
& $-\exp(G/\sqrt{T})$
  & $-\sqrt{T} \xs \log(-\sqrt{T} \xs) + \sqrt{T}\xs$
  & $\rightarrow \sqrt{e}$  & \eqr{afsalg}\\
\hline
\end{tabular}
\caption{Summary of online linear games considered in this
  paper. Results are stated for the one-dimensional problem where $g_t
  \in [-1, 1]$; Corollary~\ref{cor:nd} gives an extension to $n$ dimensions.
  The benchmark $L$ is given as a function of $G = g_{1:T}$.  The
  standard notion of regret corresponds to the $L(G) = \argmin_{x \in
    [-1, 1]} g_{1:t} \cdot x = - \abs{G}$.  The benchmark
  functions can alternatively be derived from a suitable penalty
  $\Psi$ on comparator points $\xs$, so $L(G) = \argmin_x G x +
  \Psi(x)$.}\label{fig:results}

\arxivonly{\end{small}}
\end{center}
\end{figure}

\paragraph{The Minimax Value of the Game}
Given a benchmark function $L$, the minimax value of the game is
\begin{equation}\label{eq:game}
  V^T = \rep{\inf_{x_t \in \R^n}\ \sup_{g_t \in \G}}_{t=1}^T
      \quad \left(\sum_{t=1}^T g_t \cdot x_t - L(g_1, \dots, g_T)\right)
\end{equation}
where $\rep{\inf_{x_t}\ \sup_{g_t}}_{t=1}^T$ is a shorthand notation
for $\inf_{x_1} \sup_{g_1} \dots \inf_{x_T} \sup_{g_T}$.  Against a
worst-case adversary, any algorithm must incur regret at least $V^T$,
and the minimax optimal algorithm will incur regret at most $V^T$
against any adversary.  Since in this work we study minimax
algorithms, we will often use the value of the game $V^T$ as an upper
bound on Regret (as defined in \eqr{linreg}).  Generally we will not
assume our adversaries are minimax optimal.

We are also concerned with the conditional value of the game, $V_t$,
given $x_1, \dots x_t$ and $g_1, \dots g_t$ have already been played.
That is, the Regret when we fix the plays on the first $t$ rounds, and
then assume minimax optimal play for rounds $t+1$ through $T$.
However, following the approach of \citet{rakhlin12relax}, we omit the
terms $\sum_{s=1}^t x_s \cdot g_s$ from \eqr{game}.  We can view this
as cost that the learner has already payed, and neither that cost nor
the specific previous plays of the learner impact the value of the
remaining terms in \eqr{linreg}.  Thus, we define
\begin{equation} %
  V_t(g_1, \dots, g_t)
   =  \rep{\inf_{x_s \in \R^n}\ \sup_{g_s \in \G}}_{s=t+1}^T
   \quad \left(\sum_{s=t+1}^T g_s \cdot x_s - L(g_1, \dots, g_T)\right).
\end{equation}
Note the conditional value of the game before anything has been
played, $V_0()$, is exactly $V^T$.

\paragraph{Related Work}
Regret-based analysis has received extensive attention in recent
years; see \citet{shalev12online} and \citet{cesabianchi06plg} for an
introduction.
The analysis of alternative notions of regret is also not new.  In the
expert setting, there has been much work on tracking a shifting
sequence of experts rather than the single best expert; see
\citet{koolen12shifting} and references therein.
\citet{zinkevich03giga} considers drifting comparators in an online
convex optimization framework.  This notion can be expressed by an
appropriate $L(g_1, \dots, g_T)$, but now the order of the gradients
matters, unlike the benchmarks $L$ considered in this work.
\citet{merhav02memory} and \citet{dekel12adaptive} consider the
stronger notion of policy regret in the online experts and bandit
settings, respectively.  For investing scenarios,
\citet{agarwal06portfolio} and \citet{hazan09investing} consider
regret with respect to the best constant-rebalanced portoflio.

More recently, the field has seen minimax approaches to online
learning.  \citet{abernethy10repeated} give a minimax strategy for
several zero-sum games against a budgeted adversary.
Section~\ref{sec:hypercube} studies the online linear game of
\citet{abernethy08} under different assumptions, and we adapt some
techniques from \citet{abernethy09minimax}.  \citet{rakhlin12relax}
takes powerful tools for non-constructive analysis of online learning
problems and shows they can be used to design algorithms; our work
differs in that we focus on cases where the exact minimax strategy can
be computed.

\section{Alternative Notions of Regret} \label{sec:regret}
One of our contributions is showing that that interesting results can
be obtained by choosing $L$ differently than in~\eqr{stdL}; in
particular, we obtain minimax optimal algorithms for the problem
considered by~\citet{streeter12unconstrained} by analyzing an
appropriate choice of $L$.

One could choose $L(G) = 0$, but this leads to an uninteresting game:
the adversary has no long-term constrains, and so can simply pick
$g_t$ to maximize $g_t x_t$ for whatever $x_t$ the player selected.
Thus, the player can do no better than always picking $x_t = 0$.  This
is exactly the reason for studying the standard notion of regret: we
do not require that we do well in absolute terms, but rather relative
to the best strategy from a fixed set.

That is, interesting games result when the player accepts the fact
that it is impossible to do well in terms of the absolute loss $\sum_t
g_t \cdot x_t$ for all sequences $g_1, \dots, g_T$.  However, the
player can do better on some sequences at the expense of doing worse
on others.  The benchmark function $L$ makes this notion precise:
sequences for which $L(g_1, \dots, g_T)$ is large and negative are
those on which the player desires good performance,\footnote{It can be
  useful to think about $-L(G)$ as the benchmark reward for the
  sequence with gradient sum $G$.} at the expense of allowing more
loss (in absolute terms) on sequences where $L(g_1, \dots, g_T)$ is
large and positive.  The value of the game $V^T$ tells us to what
extent any online algorithm can hope to match the benchmark
performance $L$.  It follows by definition that if we add a constant
$k$ to $L$ (making $L$ easier to achieve), we decrease the minimax
value of the game by $k$, without changing the minimax optimal
strategy.

We can use these ideas to derive algorithms for a setting that is
quite different from typical online convex optimization.  On each
round $t$, the world (possibly adversarial, possibly not) offers the
player a betting opportunity on a binary outcome; the player can take
either side of the bet.  The player begins with $\$1$, but on later
rounds can wager up to whatever amount he currently has (based on
previous wins and losses). The player selects an amount $x_t$ to bet,
and then the world reveals whether the bet was won or lost; if the
player won the bet, he receives $x_t$ dollars; otherwise, he loses
$x_t$ dollars.
The players net winnings are $-\sum_t g_t x_t$, where $g_t \in
\set{-1, 1}$; the player wins the bet when $\sign(x_t) \ne g_t$
(thus, the player strives to minimize $\sum_t g_t x_t$).  How should
the player bet in this game?  Clearly if the world is adversarial, we
cannot do better than always betting $x_t = 0$.  But, we might have
reason to believe the world is not fully adversarial; if we knew $g_t
= 1$ with a fixed probability $p$, then following a Kelly betting
scheme~\citep{kelly56betting} might be appropriate, but knowing $p$ is
often unrealistic in practice.

If the player is familiar with online linear optimization, he might
try projected online gradient descent~\citep{zinkevich03giga} with a
constant step size.\footnote{Any other algorithm that provides a bound
  on standard regret of $\BO(B \sqrt{T})$ will behave similarly.} If
we restrict our bets to the feasible set $[-B, B]$, letting
$G = g_{1:T}$, this algorithm guarantees $\Regret = \Loss + B \abs{G}
\leq 2 B \sqrt{T}$.  Then $\Winnings = -\Loss \ge B \abs{G} - 2 B
\sqrt{T}$.  Thus in the best case (when $\abs{G} = T$) the player ends
up with a little less than $B T$; but he can lose up to $B\sqrt{T}$
when $G=0$.  Thus, to ensure he loses no more than the $\$1$ he has on
hand, he must choose $B = 1/\sqrt{T}$.  With this restriction,
in the best case the player wins less than $\sqrt{T}$ dollars.
However, the post-hoc optimal strategy would have been to bet
everything every round, netting winnings of $2^T$.  Despite the
theoretical guarantees, the player certainly might feel regret at
having won only $\sqrt{T}$ in this situation!

One might also hope to use online algorithms for portfolio management,
for example those of \citet{hazan09investing}
and \citet{agarwal06portfolio}.  However, these algorithms require the
assumption that you always retain at least an $\alpha > 0$ fraction of
your bet, which is directly violated in our game.

By carefully crafting a suitable benchmark function $L$, we can
provide the player with a more satisfying algorithm.  Ideally, we
would like an $L$ that satisfies three properties: 1) there exists an
algorithm where regret is bounded by a constant $\eps$ (for any $T$)
with respect to $L$, 2) $-L(G) \geq 0$, and 3) $-L(G)$ grows
exponential in $\abs{G}$.  Properties 1) and 2) ensure the player
never loses more than $\eps$ running this algorithm; by scaling the
bets the algorithm suggests by $1/\eps$, he can ensure he never loses
more than his starting $\$1$.  Property 3 implies that for
``easy'' sequences, we get exponential reward; in fact, given 1) and
2) we would like $-L(G)$ to grow as quickly as possible.

Of course, if the adversary chooses $g_t$ uniformly at random from
$\{-1, 1\}$ each round, we expect to frequently see $\abs{G} \ge
\sqrt{T}$, and so intuitively we will not be able to guarantee
exponential winnings.  This suggests the best we might hope
for is a function like $L(G) = -\bef{\abs{G}}{\sqrt{T}}$.  In
fact, in Section~\ref{sec:inclr} we show that constant regret against
such a benchmark function is possible, and we derive a minimax
algorithm.

\paragraph{A Comparator Set Interpretation}
The classic definition of regret defines $L$ indirectly as the loss of
the best strategy from a fixed class $\X$ in hindsight, \eqr{stdL}.
As this work shows, it can be advantageous to state $L$ as an explicit
function of $G$; however, useful intuition can be gained by
interpreting $L$ as a penalty function on comparator points $\xs$.
That is, we wish to find a $\Psi$ such that
\[
  L(G) = \argmin_x G x + \Psi(x).
\]
For the benchmark functions $L$ we consider, we also derive the
corresponding penalty functions $\Psi$ using convex conjugates.  These
are summarized in our results in Figure~\ref{fig:results}.

The standard notion of regret correspond to a hard penalty $\Psi(x) =
I(x \in \X)$.  Such a definition makes sense when the player by
definition must select a strategy from some bounded set, for example a
probability from the $n$-dimensional simplex, or a distribution on
paths in a graph.  For such problems, standard regret is really
comparing the player's performance to that of any fixed feasible
strategy chosen with knowledge of $g_1, \dots, g_T$; by putting an
equal penalty on each of them, we do not indicate any prior belief
that some strategies are more likely to be optimal than others.

However, in contexts such as machine learning where any $x \in \R^n$
corresponds to a valid model, such a hard constraint is difficult to
justify; while any $x \in \R^n$ is technically feasible, in order to
prove regret bounds we compare to a much more restrictive set.  As an
alternative, in Sections~\ref{sec:gd} and \ref{sec:inclr} we propose
soft penalty functions that encode the belief that points near the
origin are more likely to be optimal (we can always re-center the
problem to match our beliefs in this reguard), but do not rule out
any $x \in \R^n$ a priori.

\section{General Unconstrained Linear Optimization}
\label{sec:gen}

In this section we prove a theorem that greatly simplifies the task of
computing minimax values and deriving algorithms for the games we
consider.  We prove this result in the one-dimensional case;
Corollary~\ref{cor:nd} then extends the result to $n$-dimensions.

\begin{theorem}\label{thm:oned}
  Consider the one-dimensional unconstrained game where the player
  selects $x_t \in \R$ and the adversary chooses $g_t \in \G = [-1,
  1]$, and $L$ is concave in each of its arguments and bounded below
  on $\G^T$.  Then,
  \begin{align*}
    V^T &= \E_{g_t \sim \set{-1, 1}}  \big[-L(g_1, \dots, g_T)\big].
  \end{align*}
  where the expectation is over each $g_t$ chosen independently and
  uniformly from $\set{-1, 1}$ (that is, the $g_t$ are Rademacher
  random variables).  Further, the conditional value of the game is
  \begin{equation} \label{eq:condv}
  V_t(g_1, \dots, g_t) = \E_{g_{t+1}, \dots, g_T \sim \set{-1, 1}}
      \big[ -L(g_1, \dots, g_t, g\ti, \dots g_T)\big].
  \end{equation}
\end{theorem}

\begin{proof}
  We argue by backwards induction (from $t=T$ to $t=1$) on the
  conditional value of the game, with the induction hypothesis that
  \begin{equation}\label{eq:ih}
  V_t(g_1, \dots, g_t)
    = \E_{g_{t+1}, \dots, g_T \sim \set{-1, 1}}[ -L(g_1, \dots, g_T)],
  \end{equation}
  and further that $V_t$ is convex in each of its arguments and
  bounded above on $\G^T$.  The induction hypothesis holds
  trivially for $T=t$, using the assumption that $L$ is concave and
  bounded below for the second part.  Now, suppose the induction
  hypothesis holds for $t$.  We then have (by the definition of $V_t$)
  \begin{equation*}
    V_{t-1}(g_1, \dots, g_{t-1})
      = \inf_{x_t}\, \sup_{g_t}\ g_t x_t + V_t(g_1, \dots, g_{t-1}, g_t).
  \end{equation*}
  Note $V_{t-1}$ must be convex in each of it's arguments, using the
  induction hypothesis on $V_t$.  Let $M(g, x) = g x + V_t(g_1, \dots,
  g_{t-1}, g)$.  We would like to appeal to the minimax theorem to
  switch the $\inf$ and $\sup$, but since $M$ is convex in $g$ (using
  the induction hypothesis) rather than concave, we cannot do so
  immediately.  However, because we are choosing $g_t \in [-1, 1]$, it
  follows from the convexity of $M$ that the supremum is obtained at
  either $-1$ or $+1$.  Thus, we can write
  \begin{align*}
    V_{t-1}(g_1, \dots, g_{t-1})
      &= \inf_{x_t}\, \sup_{g_t \in [-1, 1]}\ M(g_t, x_t) \\
      &= \inf_{x_t}\, \sup_{g_t \in \set{-1, 1}}\ M(g_t, x_t) \\
      &= \inf_{x_t}\, \sup_{p_t \in \Delta(\set{-1, 1})}\
         \E_{g_t \sim p_t} [M(g_t, x_t)],
         \intertext{where $p_t \in [0,1]$ is the probability the
           adversary chooses $g_t = +1$ (otherwise, $g_t = -1$).  Now
           $\E_{g_t \sim p_t} [M(g_t, x_t)]$ is linear in both $p_t$
           and $x_t$, and so we can apply the minimax theorem (e.g.,
           Theorem 7.1 from~\citet{cesabianchi06plg}), which gives}
   V_{t-1}(g_1, \dots, g_{t-1})
     &= \sup_{p_t \in \Delta(\set{-1, 1})}\, \inf_{x_t} \
         \E_{g_t \sim p_t} [g_t x_t + V_t(g_1, \dots, g_{t-1}, g)] \\
     &= \sup_{p_t \in \Delta(\set{-1, 1})}\, \inf_{x_t} \
         \E_{g_t\sim p_t} [g_t x_t]
         + \E_{g_t \sim p_t}[V_t(g_1, \dots, g_{t-1}, g_t)].
         \intertext{Now, the adversary (sup player) must choose $p_t =
           0.5$ so $\E[g_t] = 0$, or otherwise the player can choose
           $x_t$ to drive the value to $-\infty$ (since $V_t$ is
           bounded above).  Thus, the first expectation term
           disappears, and the choice of the player becomes
           irrelevant, giving} V_{t-1}(g_1, \dots, g_{t-1}) &=
         \E_{g_t}[V_t(g_1, \dots, g_{t-1}, g_t)],
   \end{align*}
   where now the expectation is on $g_t$ drawn i.i.d. from $\set{-1,
     1}$.  Applying the induction hypothesis completes the proof,
   since then iterated expectation yields \eqr{ih} for $V_{t-1}$, and
   boundedness is immediate.
\end{proof}
The use of randomization to allow the application of the minimax
theorem is similar to the technique used
by~\citet{abernethy09minimax}.

A key insight from the proof is that an optimal adversary can always
select from $\spm$.  With this knowledge, we can view the game as a
binary tree of height $T$.  An algorithm for the player simply assigns
a play $x \in \R$ to each node, and the adversary chooses which
outgoing edge to take: if the adversary chooses the left edge, the
player suffers loss $x$, otherwise the player wins $x$ (suffers loss
-$x$).  Finally, when leaf $\ell$ is reached, the adversary pays the
player some amount $L(\ell)$.  Theorem~\ref{thm:oned} implies the
value of the game is then simply the average value of $-L(\ell)$.

Given Theorem~\ref{thm:oned}, and the fact that the functions $L$ of
interest will generally depend only on $g_{1:T}$, it will be useful to
define $\PB_T$ to be the distribution of $g_{1:T}$ when each $g_t$ is
drawn independently and uniformly from $\set{-1, 1}$ (that is, the sum
of $T$ Rademacher random variables).

Theorem~\ref{thm:oned} immediately yields bounds for games in
$n$-dimensions where the adversary is constrained to play
$\norm{g_t}_\infty \leq 1$:
\begin{corollary}\label{cor:nd}
  Consider the game where the player chooses $x_t \in \R^n$, and the
  adversary chooses $g_t \in [-1, 1]^n$, and the total payoff is
  \[
  \sum_{t=1}^T g_t \cdot x_t - \sum_{i=1}^n L(g_{1:T, i})
  \]
  for a concave function $L$.  Then, the value of the game is
  \[
    V^T = n \E_{G \sim \PB_T} \big[-L(G)\big],
  \]
  Further, the conditional value of the game is
  \begin{equation*}
    V_t(g_1, \dots, g_t) =
    \sum_{i=1}^n \E_{G_i \sim \PB_{T - t}} \big[ -L(g_{1:t,i} + G_i)\big].
  \end{equation*}
\end{corollary}
\begin{proof}[Proof sketch.]
  The proof follows by noting the constraints on both players'
  strategies and the value of the game fully decompose on a
  per-coordinate basis.
\end{proof}

\paragraph{A recipe for minimax optimal algorithms in one dimension}
For any function $L$,
\begin{equation}\label{eq:binomG}
 \E_{G \sim \PB_T}[L(G)] = \frac{1}{2^T}\sum_{i=0}^T  \binom{T}{i}L(2i - T),
\end{equation}
since $2^{-T}\binom{T}{i}$ is the binomial probability of getting
exactly $i$ gradients of $+1$ over $T$ rounds, which implies $T-i$
gradients of $-1$, so $G = i - (T-i) = 2i - T$.

Since \eqr{condv} gives the minimax value of the game if both players
play optimally from round $t+1$ forward, a minimax strategy for the
learner on round $t+1$ must be
\begin{align}
x_{t+1} &= \argmin_{x\in\R} \max_{g \in \set{-1, 1}} \ \
             g \cdot x + V\ti(g_1, \dots, g_t, g) \notag \\
        &= \h \big( V\ti(g_1, \dots, g_t, -1) - V\ti(g_1, \dots, g_t, +1) \big).
           \label{eq:valg}
\end{align}
The second line follows because the argmin is simply over the max of
two intersecting linear functions, which we can compute in closed form
as the point of intersection.
Thus, if we can derive a closed form for $V_t(g_1, \dots, g_t)$, we
will have an efficient minimax-optimal algorithm.  In the next
section, we explore cases where this is possible.

When $L$ depends only on $G = g_{1:T}$, we may be able to run the
minimax algorithm efficiently even if $V_t$ does not have a convenient
closed form: if $\tau = T -t$, the number of rounds remaining, is
small, then we can compute $V_t$ exactly by using the appropriate
binomial probabilities (following~\eqr{condv} and \eqr{binomG}).  On
the other hand, if $\tau$ is large, then applying the Gaussian
approximation to the binomial distribution may be sufficient.

\section{Deriving Minimax Optimal Algorithms} \label{sec:algs}

In this sections, we explore three applications of the tools from the
previous section.  We begin with a relatively simple but interesting
example which illustrates the technique.

\subsection{Constant step-size gradient descent can be minimax
  optimal}\label{sec:gd}
Suppose we use a ``soft'' feasible set for the benchmark,
\begin{align}
L(G) &= \min_{x}\  G x + \frac{\sigma}{2} x^2
      = -\frac{1}{2 \sigma} G^2,  \label{eq:qopt}
\end{align}
for a constant $\sigma > 0$.  Does a no-regret algorithm against this
comparison class exist?  Unfortunately, the general answer is no, as
shown in the next theorem:

\begin{theorem}\label{thm:csg}
The value of this game is
$
  V^T = \E_{G \sim \PB_T}\Big[\frac{1}{2 \sigma} G^2\Big]
    = \frac{T}{2\sigma}.
$
\end{theorem}
\begin{proof}
Starting from~\eqr{binomG},
\begin{align*}
  \E_{G \sim \PB_T}[G^2]
    &= \frac{1}{2^T} \sum_{i=0}^T \binom{T}{i}(2i - T)^2
       && \eqr{binomG}\\
    &= \frac{1}{2^T} \left(
         4\sum_{i=0}^T \binom{T}{i} i^2
         -4T\sum_{i=0}^T \binom{T}{i} i
         +T^2 \sum_{i=0}^T\binom{T}{i} \right) \\
\intertext{and since $\sum_{t=0}^T\binom{T}{t}=2^T$,
  $\sum_{t=0}^T\binom{T}{t} t = T 2^{T-1}$,
  $\sum_{t=0}^T\binom{T}{t} t^2 =  (T+T^2)2^{T-2}$,}
    &= \frac{1}{2^T} \Big(
         4 (T+T^2)2^{T-2}
         -4T (T2^{T-1})
         +T^22^T \Big) \\
    &=  (T+T^2) -2T^2 + T^2 = T.
\end{align*}
The result then follows from linearity of expectation.
\end{proof}
This implies $\Reward \ge -L(G) - \Regret = \frac{1}{2\sigma}\big(G^2
- T)$, a fact noted by ~\citet[Lemma 2]{streeter12unconstrained}.

Thus, for a fixed $\sigma$, we cannot have no a regret algorithm with
respect to this $L$.  However, if $T$ is known in advance, we could
choose $\sigma = \sqrt{T}$ in order to claim no-regret.  But this is a
bit arbitrary:
if the player could pick $\sigma$, and cares purely about Regret,
obviously he would like to play the game where $\sigma \rightarrow
\infty$, as that makes the value of the game (Regret) as small as
possible.  However, this choice also drives Reward to zero.  If the
lower-bound on reward is what matters, then the player should choose
based on how he expects $G^2$ to relate to $T$.

To derive the minimax optimal algorithm, we can compute conditional
values (using similar techniques to Theorem~\ref{thm:csg}),
\[
  V_t(g_1, \dots, g_t)
  = \E_{G \sim \PB_{T-t}}\Big[\frac{1}{2 \sigma} (g_{1:t} + G)^2\Big]
  = \frac{1}{2\sigma}\big((g_{1:t})^2 + (T-t)\big),
\]
and so following \eqr{valg} the minimax-optimal algorithm must use
\begin{align*}
x\ti
 &= \frac{1}{4\sigma} \left(
   \big((g_{1:t} - 1)^2 + (T-t-1)\big) - ((g_{1:t} + 1)^2 + (T-t-1)) \right) \\
 &= \frac{1}{4\sigma} ( -4g_{1:t})
  = -\frac{1}{\sigma}g_{1:t}
\end{align*}
Thus, a minimax-optimal algorithm is simply constant-learning-rate
gradient descent with learning rate $\frac{1}{\sigma}$.  Note that for
a fixed $\sigma$, this is the optimal algorithm independent of $T$;
this is atypical, as usually the minimax optimal algorithm depends on
the horizon (as we will see in the next two cases).

\subsection{Optimal regret against hypercube adversaries}
\label{sec:hypercube}
\citet{abernethy08} gives a minimax optimal algorithm when the
player's $x_t$ and the comparator $\xs$ are constrained to an $L_2$
ball, and the adversary must also select $g_t$ from an $L_2$ ball, for
$n \ge 3$ dimensions.\footnote{Their results are actually more general
  than this, allowing the constraint on $\norm{g_t}_2$ to vary on a
  per-round basis.  Our work could also be extended in that manner.}
In contrast, we consider regret compared to the best $\xs$ constrained
to the unit $L_\infty$ ball, but allow the player to select any $x_t
\in \R^n$; our adversary is constrained to select $g_t$ from the unit
$L_\infty$ ball (the generalization to arbitrary hyper-rectangles is
straightforward).
Perhaps surprisingly, the optimal strategy for the player always plays
from the unit $L_\infty$ ball as well, so our results immediately
apply to the case of the constrained player.

Since we consider $L_\infty$ constraints on both the comparator and
adversary, Corollary~\ref{cor:nd} implies it is sufficient to study
the one-dimensional case.  We consider the standard notion of regret,
taking $L(G) = -\abs{G}$ following \eqr{stdL}.  Our main result is the
following:

\begin{theorem}\label{thm:onedregret}
  Consider the game between an adversary who chooses loss functions
  $g_t \in [-1, 1]$, and a player who chooses $x_t \in \R$.  For a
  given sequence of plays, $x_1, g_1, x_2, g_2, \dots, x_T, g_T$, the
  value to the adversary is
  $
  \sum_{t=1}^T g_t x_t - \abs{g_{1:T}}.
  $
  Then, when $T$ is even with $T = 2M$, the minimax value of this game
  is given by
  \[
    V_T =  2^{-T} \frac{2 M\, T!}{(T - M)!M!}
       \leq  \sqrt{\frac{2T}{\pi}}.
  \]
  Further, as $T \rightarrow \infty$, $V_T \rightarrow
  \sqrt{\frac{2T}{\pi}}$.
\end{theorem}

\begin{proof}  Letting $T = 2M$ and working from~\eqr{binomG},
\begin{equation}\label{eq:linoptonedval}
  V^T = -\E_{G \sim \PB_T}[L(G)]
=  \frac{2}{2^T}\sum_{i=0}^T \binom{T}{i}\abs{i - M}
= \frac{2M}{2^T} \binom{2M}{M}
= 2^{-T} \frac{2 M\, T!}{(T - M)!M!},
\end{equation}
where we have applied a classic formula of ~\citet{demoivre1718} for
the mean absolute deviation of the binomial distribution (see also
\citet{diaconis91}).  Using a standard bound on the central binomial
coefficient (based on Stirling's formula),
\begin{equation}\label{eq:centralbinbound}
 \binom{2M}{M} = \frac{4^M}{\sqrt{\pi M}}\left(1 - \frac{c_M}{M}\right)
\end{equation}
where $\frac{1}{9} < c_M < \frac{1}{8}$ for all $M \geq 1$, we have
\[
V^T \leq 2 M \frac{1}{\sqrt{\pi M}} = \sqrt{\frac{2T}{\pi} }.
\]
As implied by \eqr{centralbinbound}, this inequality quickly becomes
tight as $T \rightarrow \infty$.
\end{proof}

\paragraph{The minimax algorithm (for the constrained player, too!)}
In order to compute the minimax algorithm, we would like a closed form
for
\[
  V_t(G_t) = -\E_{G^\tau \sim \PB_{\tau}} \big[L(G_t + G^\tau)\big],
\]
where $G_t = g_{1:t}$ is the sum of the gradients so far, $\tau = T
-t$ is the number of rounds to go, and and $G^\tau = g_{t+1:T}$ is a
random variable giving the sum of the remaining gradients.
Unfortunately, the structure of the binomial coefficients exploited by
de Moivre and used in \eqr{linoptonedval} does not apply given an
arbitrary offset $G^\tau$.  Nevertheless, we will be able to derive a
formula for the update that is readily computable.
Writing $\Pr_\tau(b)$ for the probability a random draw from
$\PB_\tau$ has value $b$, the update of ~\eqr{valg} becomes
\begin{align}
x_{t+1} &= \h \sum_{b=-\tau}^\tau \Pr_\tau(b)
  \Big(\abs{G_t + b -1} - \abs{G_t + b + 1}\Big). \notag \\
\intertext{Whenever $G_t + b \geq 1$, the difference in
  absolute values is $-2$, and whenever $G_t + b \leq 1$, the
  difference is $2$.  When $G_t + b = 0$, the difference is zero.
  Thus,}
x_{t+1} &= \h\left(\Pr_\tau(b > - G)(-2) + \Pr_\tau(b < -G) (2) \right)\notag \\
        &= \Pr_\tau(b < -G) - \Pr_\tau(b > - G). \label{eq:salg}
\end{align}
While this update does not have a closed form, it can be efficiently
computed numerically.\footnote{The CDF of the binomial
  distribution can be computed numerically using the regularized
  incomplete beta function, from which $\Pr_\tau(b \le -G)$ can be
  derived.  Then, $\Pr_\tau(b = -G)$ can be computed from the
  appropriate binomial coefficient, leading to both needed
  probabilities.}
It follows from this expression that even though we allow the player
to select $x\ti \in \R$, the minimax optimal algorithm always selects
points from $[-1,1]$. Thus, we have the following Corollary:
\begin{corollary}
  Consider the game of Theorem~\ref{thm:onedregret}, but suppose now
  we also constrain the player to choose $x_t \in [-1, 1]$.  This does
  not change the value of the game, as the minimax algorithm for the
  unconstrained case always plays from $[-1, 1]$ regardless.
\end{corollary}

\citet{abernethy08} shows that for the linear game with $n \geq 3$
where both the learner and adversary select vectors from the unit
sphere, the minimax value is exactly $\sqrt{T}$.  Interestingly, in
the $n=1$ case (where $L_2$ and $L_\infty$ coincide), the value of the
game is lower, about $0.8 \sqrt{T}$ rather than $\sqrt{T}$.  This
indicates a fundamental difference in the geometry of the $n=1$ space
and $n \geq 3$.  We conjecture the minimax value for the $L_2$ game
with $n=2$ lies somewhere in between.

\subsection{Non-stochastic betting and No-regret for all feasible sets
  simultaneously}
\label{sec:inclr}

We derive a minimax optimal approach to the betting problem presented
in Section~\ref{sec:regret}, which also corresponds to the setting
introduced by~\citet{streeter12unconstrained}.  Again, it is
sufficient to consider the one-dimensional case.  In that work, the
goal was to prove bounds like $\Regret \leq \BO(R \sqrt T
\log\p{(1+R)T})$ simultaneously for any comparator $\xs$ with
$\abs{\xs} = R$.  Stating their Theorem~1 in terms of losses, this
bound is achieved by any algorithm that guarantees
\begin{equation}\label{eq:incpot}
  \text{Loss} = \sum_{t=1}^T g_t x_t
    \leq - \bexp{\frac{\abs{g_{1:T}}}{\sqrt{T}}} + \BO(1).
\end{equation}
Note that whenever $\abs{g_{1:T}}$ is large compared to $\sqrt{T}$ the
algorithm must achieve significantly negative loss (positive reward).

We initially study the game where
\begin{equation}\label{eq:Lexp}
 L(G) = - \bexp{\frac{G}{\sqrt{T}}}
\end{equation}
(note $G = g_{1:T} \in [-T, T]$ can be positive or negative).  We
prove the minimax algorithm achieves $\sum_{t=1}^T g_t x_t -
L(g_{1:T}) \leq \sqrt{e}$, implying $ \Reward = -\sum_{t=1}^T g_t x_t
\geq \bef{G}{\sqrt{T}} - \sqrt{e}$.  Thus, this algorithm guarantees
large reward whenever the gradient sum $G$ is large and positive.  In
order to satisfy \eqr{incpot}, we must also achieve large reward
when $G$ is large and negative.  Since $L(g_{1:t}) + L(-g_{1:t}) \leq
-\bexp{\frac{\abs{g_{1:T}}}{\sqrt{T}}}$, this can be accomplished by
running two copies of the minimax algorithm simultaneously, switching
the signs of the gradients and plays of the second copy.  We formalize
this in Appendix~\ref{ap:sym}.

\paragraph{Interpretation as a soft feasible set}
Before developing an algorithm it is worth noting an alternative
characterization of this benchmark function.  One can show, that for
$a \ge 0$,
\[
  \min_{x \in \R^-} \paren{G x - a x \log( -a x) + ax}
   = -\bef{G}{a}
\]
Thus, if we take $\Psi(x) = - a x \log( a x) + ax + I{(x \le 0)}$, we have
\[
 \min_{x \in \R^-} g_{1:T} x + \Psi(x) = -\bef{G}{a}.
\]
Since this algorithm needs large Reward when $G$ is large and
positive, we might expect that the minimax optimal algorithm only
plays $x_t \leq 0$.  Another intuition for this is that the algorithm
should not need to play any point $\xs$ to which $\Psi$ assigns an
infinite penalty.  This intuition is confirmed by the analysis of this
``one-sided'' algorithm:

\begin{theorem}\label{thm:onesided}
  Consider the game with benchmark $L$ as defined in \eqr{Lexp}.
  The minimax value of this game is exactly
  \[
  V^T = \frac{\left(1 + \exp\pfrac{2}{\sqrt{T}}\right)^T}
         {2^{T}\exp\big(\!\sqrt{T}\big)} \leq \sqrt{e},
  \]
  and further $\lim_{T \rightarrow \infty} V^T = \sqrt{e}$.  Letting
  $\tau = T - t$ be the number of rounds left to be played, and
  defining $G_t = g_{1:t}$, the conditional value of the game is
  \begin{equation*}
    V_t(G_t) = 2^{-\tau}\bexp{\frac{G_t-\tau}{\sqrt{T}}}
    \Big(1 + \exp\big(2/\sqrt{T}\big)\Big)^{\tau},
  \end{equation*}
  which leads to the minimax optimal algorithm\footnote{When computing
    the player's strategy via \eqr{afsalg}, it is numerically
    preferable to do the calculation in log-space, and then
    exponentiate to get the final play.} for the player
  \begin{equation}\label{eq:afsalg}
  x\ti = -2^{-\tau} \bef{G_t -\tau -1}{\sqrt{T}}
     \left(\bef{2}{\sqrt{T}} - 1\right)
     \left(\bef{2}{\sqrt{T}} + 1\right)^\tau \le 0.
  \end{equation}
\end{theorem}

\begin{proof}
First, we compute the value of the game:
\begin{align*}
V^T = \E_{G \sim \PB_T} \big[-L(G)\big]
  &= 2^{-T}\sum_{i=0}^T \binom{T}{i} \bexp{\frac{2i - T}{\sqrt{T}}}\\
  &= 2^{-T}\exp\big(-\!\sqrt{T}\big)
        \sum_{i=0}^T \binom{T}{i} \Big(\exp\big(2/\sqrt{T}\big)\Big)^i \\
  &= 2^{-T}\exp\big(-\!\sqrt{T}\big)\Big(1 + \exp\big(2/\sqrt{T}\big)\Big)^T,
\end{align*}
where we have used the ordinary generating function, $ \sum_{i=0}^T
\binom{T}{i} x^i = (1 + x)^T$.  Manipulating the above expression for
the value of the game, we arrive at $V^T = \cosh(1/\sqrt{T})^T$.
Using the series expansion for $\cosh$ leads to the upper bound
$\cosh(x) \le \exp(x^2/2)$,
from which we conclude
\[
 V_T = \left(\cosh\big(1/\sqrt{T}\big)\right)^T
  \le \bef{1}{2T}^T = \sqrt{e}.
\]

Using similar techniques, we can derive the conditional value of the
game, letting $\tau = T - t$ be the number of rounds left to be
played:
\begin{align*}
V_t(G_t)
  &= 2^{-\tau}\sum_{i=0}^{\tau} \binom{\tau}{i}
  \bexp{\frac{G_t + 2i - \tau}{\sqrt{T}}}
  = 2^{-\tau}\bexp{\frac{G_t-\tau}{\sqrt{T}}}
        \Big(1 + \exp\big(2/\sqrt{T}\big)\Big)^{\tau}.
\end{align*}
Following ~\eqr{valg} and simplifying leads to the update of
\eqr{afsalg}.
It remains to show $\lim_{T \rightarrow \infty} V_T = \sqrt{e}$.
Using the change of variable $x = 1/\sqrt{T}$, equivalently we
have $\lim_{x \rightarrow 0} \cosh(x)^{\frac{1}{x^2}}$.  Examining the
log of this function,
\begin{align*}
  \lim_{x \rightarrow 0}  \log\paren{\cosh(x)^{\frac{1}{x^2}}}
  & = \lim_{x \rightarrow 0} \frac{1}{x^2} \log \cosh(x)
   = \lim_{x \rightarrow 0} \frac{1}{x^2}
  \paren{\frac{x^2}{2}  - \frac{x^4}{12}
    + \frac{x^6}{45} -\frac{17x^8}{2520} + \dots}
  = \h,
\end{align*}
where we have taken the Maclaurin series of $\log \cosh (x)$.
Using the continuity of $\exp$, we have
\[
\lim_{x \rightarrow 0} \paren{\cosh(x)^{\frac{1}{x^2}}}
= \exp \paren{\lim_{x \rightarrow 0}  \log\paren{\cosh(x)^{\frac{1}{x^2}}}}
= \sqrt{e}.
\]
\end{proof}

\paragraph{A strong lower-bound}
\newcommand{\La}{L_\alpha}
Recall from Section~\ref{sec:regret} that as long as $-L(G) \ge 0$ and
we get constant regret with respect to $L$, we can scale our bets so
that we never risk losing more than a constant starting budget.  This
holds for any number of rounds $T$ against any adversary.
Given that constraint, we would like $-L(G)$ to grow as fast as
possible, so it is natural to consider the generalizing \eqr{Lexp} as
\[
 \La(G) = -\bexp{\frac{G}{T^\alpha}}
\]
for $\alpha \in (0, \h]$.  Following the techniques used in the
preceding proof, we can show for this game
\[
 V^T_\alpha
 = \E_G[L(G)] = 2^{-T}\bexp{-T^{1 - \alpha}}\left(1 + \bexp{2T^{-\alpha}}\right)^T
 = \cosh\big(T^{-\alpha}\big)^T.
\]
By taking the first term in the series for $\log \cosh x$, namely
$x^2/2$, and plugging in $x = 1/T^\alpha \leq 1$, we get a good upper
bound on the value of the game:
\[
V^T
 = \exp\big( T \log \cosh(T^{-\alpha})\big)
 \le \bexp{T \frac{1}{2T^{2\alpha}}}
     =  \bexp{\h T^{1 - 2\alpha}}
\]
This implies that, for any $\alpha < 1/2$, no algorithm can provide
constant loss (that is, $\sum_{t=1}^T g_t x_t \leq k$ for a constant
$k \geq 0$) for any sequence while also guaranteeing
\begin{equation}\label{eq:bigreward}
\Reward = -\sum_{t=1}^T g_t x_t = \Omega \left(\bexp{\frac{G}{T^\alpha}}\right)
\end{equation}
for any $\alpha < 1/2$.  In fact, for $\alpha < 1/2$, no algorithm can
guarantee even linear loss in the worst case while making the reward
guarantee of \eqr{bigreward}.

\clearpage
\bibliographystyle{plainnat}
\bibliography{unconstrained_minimax}

\clearpage
\appendix
\section{A Symmetric Betting Algorithm} \label{ap:sym}

The one-sided algorithm of Theorem~\ref{thm:onesided} has
\begin{align*}
\Loss &= V^T + L(G)
      \leq - \bexp{\frac{G}{\sqrt{T}}} + \sqrt{e}.
\end{align*}
In order to do well when $g_{1:T}$ is large and negative, we can run a
copy of the algorithm on $-g_1, \dots, -g_T$, switching the signs of
each $x_t$ it suggests.  The combined algorithm then satisfies
\begin{align*}
\Loss
  &\leq -\bexp{\frac{G}{\sqrt{T}}} - \bexp{\frac{-G}{\sqrt{T}}} + 2\sqrt{e} \\
  &\leq - \bexp{\frac{\abs{G}}{\sqrt{T}}} + 2\sqrt{e},
\end{align*}
and so following \eqr{incpot} and Theorem~1 of
\citet{streeter12unconstrained}, we obtain the desired regret bounds.
The following theorem implies the symmetric algorithm is in fact
minimax optimal with respect to the combined benchmark
\[
L_C(G) = -\bef{G}{\sqrt{T}} -\bef{-G}{\sqrt{T}}.
\]

\begin{theorem}
  Consider two $1$-D games where the adversary plays from $[-1, 1]$,
  defined by concave functions $L_1$ and $L_2$ respectively.  Let
  $x_t^1$ and $x_t^2$ be minimax-optimal plays for $L_1$ and $L_2$
  respectively, given that $g_1, \dots g_{t-1}$ have been played so
  far in both games.  Then $x_1 + x_2$ is also minimax optimal for the
  combined game that uses the benchmark $L_C(G) = L_1(G) +
  L_2(G)$.
\end{theorem}
\begin{proof}
  First, taking $\tau = T - t$ and using Theorem~\ref{thm:oned} three
  times, we have
\newcommand{\EGT}{\E_{G^\tau \sim \PB_\tau}}
\begin{align*}
V^C(g_1, \dots, g_t)
  &= -\EGT\big[L_1(g_{1:t} + G^\tau)
      + L_2(g_{1:t} + G^\tau) \big] \\
  &= -\EGT\big[L_1(g_{1:t} + G^\tau)\big]
     - \EGT\big[L_2(g_{1:t} + G^\tau) \big] \\
  &= V^1(g_1, \dots, g_t) + V^2(g_1, \dots, g_t),
\end{align*}
using linearity of expectation.  Then, using \eqr{valg} for each of
the three games, we have
\begin{align*}
x_t^C
  &= \argmin_x\ \max_g\ g x + V_C(g_1, \dots, g_{t-1}, g) \\
  &= \h \big(V_C(g_1, \dots, g_{t-1}, -1) - V_C(g_1, \dots, g_{t-1}, +1) \big) \\
  &= \h \big(V_1(g_1, \dots, g_{t-1}, -1) + V_2(g_1, \dots, g_{t-1}, -1)
           - V_1(g_1, \dots, g_{t-1}, _1) - V_2(g_1, \dots, g_{t-1}, +1) \big) \\
  &= x_t^1 + x_t^2, %
\end{align*}
\end{proof}

\end{document}